\documentclass[letterpaper, 10 pt, conference]{ieeeconf}  

\IEEEoverridecommandlockouts                              

\usepackage{bm}
\usepackage{nicefrac}
\usepackage{amssymb}
\usepackage{amsmath}

\usepackage{amsthm}
\usepackage{cite}
\usepackage{xcolor}
\usepackage{graphicx}
\usepackage{url}
\usepackage{svg}
\usepackage{balance}
\usepackage{multirow}
\usepackage{subfigure}
\usepackage{hyperref}
\hypersetup{
    colorlinks=true,
    linkcolor=black,
    filecolor=magenta,      
    urlcolor=cyan,
    pdfpagemode=FullScreen,
    }
\newtheoremstyle{exampstyle}
  {3pt} 
  {3pt} 
  {\itshape} 
  {} 
  {\bfseries} 
  {.} 
  {.5em} 
  {} 

\theoremstyle{exampstyle} 
\newtheorem{definition}{Definition}

\newtheorem{theorem}{Theorem}

\newtheorem{assumption}{Assumption}
\newtheorem{problem}{Problem}
\newtheorem{example}{Example}

\theoremstyle{plain}

\definecolor{mumred}{RGB}{222,33,77}
\definecolor{mumgreen}{RGB}{0, 140, 0}
\definecolor{mumblue}{RGB}{0, 100, 222}
\definecolor{mumpurple}{RGB}{128, 0, 128}

\newcommand{\var}{$\mathrm{VaR}$ }
\newcommand{\cvar}{$\mathrm{CVaR}$}
\newcommand{\cvarbar}{$\overline{\mathrm{CVaR}}$ }
\newcommand{\ofx}{\left(\bm{x}\right)}
\newcommand{\ofb}{\left(\bm{b}\right)}
\DeclareMathOperator*{\argmin}{arg\,min}

\title{\LARGE \bf
  Risk-aware Control for Robots with Non-Gaussian Belief Spaces
}

\author{Matti Vahs and Jana Tumova
	\thanks{This work was partially supported by the Wallenberg AI, Autonomous
		Systems and Software Program (WASP) funded by the Knut and Alice
		Wallenberg Foundation. This research has been carried out as part of the Vinnova Competence Center for Trustworthy Edge Computing Systems and Applications at KTH Royal Institute of Technology. Our experiments were carried out in the WASP Research Arena
(WARA) - Robotics, hosted by ABB Corporate Research
Center in V\"aster\r as, Sweden.
  }
	\thanks{The authors are with the Division of Robotics, Perception and Learning, KTH Royal Institute of Technology, Stockholm, Sweden and also affiliated with Digital Futures. Mail addresses: {\{\tt\small vahs, tumova\}}
		{\tt\small @kth.se}}%
}

\begin{document}
	\maketitle
	\thispagestyle{empty}
	\pagestyle{empty}

	\begin{abstract}
        
		This paper addresses the problem of safety-critical control of autonomous robots, considering the ubiquitous uncertainties arising from unmodeled dynamics and noisy sensors. To take into account these uncertainties,  probabilistic state estimators are often deployed to obtain a belief over possible states. Namely, Particle Filters (PFs) can handle arbitrary non-Gaussian distributions in the robot's state. 
  In this work, we define the belief state and belief dynamics for continuous-discrete PFs and construct safe sets in the underlying belief space. We design a controller that provably keeps the robot's belief state within this safe set. As a result, we ensure that the risk of the unknown robot's state violating a safety specification, such as avoiding a dangerous area, is bounded. 
  We provide an open-source implementation as a ROS2 package and evaluate the solution in simulations and hardware experiments involving high-dimensional belief spaces. 

	\end{abstract}

\section{INTRODUCTION}
Autonomous robots are often exposed to various uncertainties in real-world applications such as unmodeled dynamics, noisy sensor readings or partially known environments. Ultimately, these uncertainties prevent us from accurately knowing the state of a robot. This requires us to consider probabilistic state estimators to obtain a robot's \emph{belief}, i.e. a probability distribution over possible states \cite{thrun2005probabilistic}.

Oftentimes, a Kalman Filter (KF) and its variations are deployed in practice. These methods  assume that the underlying probability distribution is Gaussian. Although they are reliable in many scenarios, KFs might fail in the case of highly nonlinear systems or multi-modal distributions. Particle filters (PFs) \cite{dellaert1999monte}, on the other hand, can represent an arbitrary non-Gaussian belief 
as a discrete set of samples. Consider for example the robot depicted in Fig.~\ref{fig:Intro} that is using a 2D LiDAR for localization in a given map. Due to nonlinear dynamics and observation models, the belief is non-Gaussian as illustrated by the red samples of the PF. 

Guaranteeing safety, such as avoiding dangerous areas in the map that are difficult to detect by local sensing (untraversable terrain, glass walls, ...) is challenging. This is because it requires us to consider the uncertainty in the robot's state due to the fact that we do not know exactly where the robot is located. Existing works for safety-critical control under stochastic uncertainties typically assume Gaussian beliefs \cite{blackmore2006probabilistic, van2012motion, 8613928, ono2008iterative}, bounded estimation errors \cite{Zhang_2022, Wang_2021, dean_20}, or only consider external uncertainty but also assume deterministic dynamics \cite{de2021scenario, de2023scenario}. Applying these methods to the described scenario might not yield desired safety. 

Providing rigorous safety guarantees for general nonlinear systems under stochastic uncertainties is challenging for several reasons: 1) The true underlying distribution over states is unknown as the standard Bayes Filter is intractable for continuous nonlinear systems, 2) the curse of dimensionality in belief spaces complicates the design of computationally efficient solutions and 3) hard safety constraints on the state cannot be ensured due to unbounded probability distributions. The latter point requires us to consider safety specifications in a risk-aware sense, meaning that small violations are allowed where the admissible level of risk is ultimately defined by the user. Measures of risk enable us to reason about specification violation in the presence of uncertainty \cite{nyberg2021risk, barbosa2021risk, majumdar2020should}. Especially measures of risk such as \emph{Conditional-Value-at-Risk} (\cvar) have gained a lot of attention in the robotics community as these account for less likely tail events that could result in unsafe robot behavior.

\begin{figure}[t]
    \centering
    \includegraphics[scale=0.82]{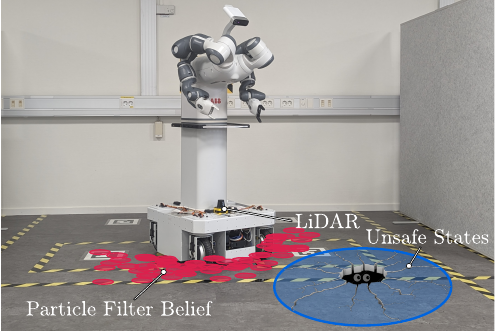}
    \caption{An illustration of our experiments with the ABB Mobile YuMi Research Platform operating in an uncertain environment. The robot is equipped with a LiDAR sensor used for localization. Due to uncertainties in motions and observations, we only have access to a belief which is provided by a PF shown as red dots. The robot is supposed to avoid a dangerous area that cannot be detected by local sensing such as e.g. an untraversable area.}
    \vspace{-0.65cm}
    \label{fig:Intro}
\end{figure}

To address the three above mentioned challenges and offer safety guarantees in real-world robotic scenarios, we propose a solution enhancing Control Barrier Functions (CBFs) to non-Gaussian belief spaces. Namely, we 
\begin{enumerate}
    \item enable guaranteed risk-awareness by using CVaR for arbitrary non-Gaussian beliefs described by PFs,
    \item provide computationally efficient control synthesis under high dimenionalities of  belief spaces,
    \item offer an open-source ROS2 package\footnote{available at \href{https://github.com/mattivahs/pfcbf\_ros2}{https://github.com/mattivahs/pfcbf\_ros2}} that can be easily integrated with existing navigation stacks.
\end{enumerate}

\section{Related Work}
There is a large body of literature on safety-critical control under bounded uncertainties. These works are based on tools from set-based reachibility analysis \cite{althoff2014online, pek2020using}, robust model predictive control (MPC) \cite{campo1987robust, langson2004robust, roque2022corridor, zeilinger2009real} or robust CBFs \cite{xu2015robustness, kolathaya2018input, Zhang_2022, Wang_2021}. While these methods have proven effective in many scenarios, it is important to consider that most estimators deployed on robots consider disturbances of stochastic nature which are typically unbounded. 

Stochastic disturbances on the system dynamics have been addressed in \cite{schwarm1999chance, clark, cosner2023robust, black2023safety, li2002probabilistically, oldewurtel2013stochastic, singletary2022safe,lew2023risk} which are based on MPC and CBF approaches. Stochastic CBFs (SCBFs) ensure safety with probability one for systems that are described by stochastic differential equations (SDEs) \cite{clark}. 
In \cite{lew2023risk}, risk measures such as \cvar~are leveraged for risk-aware trajectory optimization under stochastic dynamics. Their approach is based on sample-average approximations of risk constraints which is real-time applicable if the number of samples is low. Although \cite{schwarm1999chance, clark, cosner2023robust, black2023safety, li2002probabilistically, oldewurtel2013stochastic, singletary2022safe,lew2023risk} consider the stochastic evolution of the system dynamics, they assume precise knowledge of the state, thus neglecting the uncertainties in sensor observations.  

Approaches based on a belief instead of a state have been proposed in \cite{blackmore2006probabilistic, van2012motion, 8613928, ono2008iterative, vahs2023belief, ostafew2016robust}. Chance-constrained nonlinear MPC (CCNMPC) \cite{8613928} enforces obstacle avoidance with a desired confidence in which the uncertainty in the state estimate originates from an unscented KF. In \cite{vahs2023belief}, the authors focused on providing safety guarantees through risk-aware control problem for systems which use an extended KF as state estimator. However, all of these methods assume that the belief follows a Gaussian distribution which might result in unsafe behaviors in scenarios in which a Gaussian is a poor approximation of the true distribution. 


Only few works consider non-Gaussian beliefs for safety-critical control. In the seminal work \cite{blackmore2010probabilistic}, the authors propose a chance constrained planning framework for obstacle avoidance that leverages PFs for belief propagation and approximates chance constraints through sample averages. In \cite{platt2017efficient}, samples from a PF are used in planning to disambiguate a multiple potential hypotheses but safety constraints are not considered. Online POMDP solvers proposed in \cite{sunberg2018online} combine PFs and monte carlo tree search (MCTS) for planning under uncertainty which allows to maximize a reward function but does not include rigorous safety constraints. Additionally, to make \cite{blackmore2010probabilistic, platt2017efficient, sunberg2018online} computationally tractable, only a subset of particles from the actual PF is used. Consequently, there can be a mismatch between the sample-based distribution and the true distribution which can lead to unsafe behavior at deployment.

In this paper, we consider the problem of risk-aware control where the only thing we have access to is a PF belief. In contrast to other existing works, we consider the entire PF belief which could consist of thousands of particles and still offer a computationally effective solution. Furthermore, we account for the mismatch between the sample-based distribution and the true underlying continuous distribution by using an underapproximation of \cvar.







\section{Preliminaries}
Consider a robot that is modeled by a stochastic differential equation (SDE) of the form
\begin{align}
    \mathrm{d}\bm{x} &= \left(\bm{f}\left(\bm{x}\right) + \bm{g}\left(\bm{x}\right)\bm{u}\right) \mathrm{d}t + \bm{\sigma}\ofx ~\mathrm{d}\bm{W}\label{eq:SDE}
\end{align}
where $\bm{x} \in \mathcal{X} \subseteq \mathbb{R}^{n_x}$ is the state, $\bm{u} \in \mathcal{U} \subseteq \mathbb{R}^m$ is the control input and $\bm{W} \in \mathbb{R}^q$ is $q\text{-dimensional}$ Brownian motion. We assume that the diffusion term $\bm{\sigma}\ofx$ is a globally Lipschitz non-degenerate diagonal matrix and the drift term $\left(\bm{f}\left(\bm{x}\right) + \bm{g}\left(\bm{x}\right)\bm{u}\right)$ is a locally Lipschitz function that can have jumps. 
These assumptions ensure that Eq.~\eqref{eq:SDE} has a unique global strong solution \cite{leobacher2017strong}.

\subsection{Continuous-discrete Particle Filters}
In practice the state of a robot cannot be measured precisely as measurements are corrupted by noise. We assume that the robot's observations can be modeled by 
\begin{align}
    \bm{z}_{t_k} &= \bm{\ell}\left(\bm{x}_{t_k}, \bm{v}_{t_k}\right)\label{eq:observation}
\end{align}
with measurement noise $\bm{v}_{t_k} \sim p(\bm{v})$. Measurements only occur at discrete timesteps $t_1,\ldots, t_k$ due to limited update rates of sensors.
 Thus, the exact state at time $t$ is unknown, but a Bayesian posterior $p\left(\bm{x}_t \mid \bm{z}_{t_1:t_k}\right)$, for $t_k \leq t < t_{k+1}$  describes the probability distribution over states conditioned on past measurements. Unfortunately, this Bayesian posterior does not admit closed form solutions except for specialized cases such linear Gaussian systems. PF techniques stem from the idea of approximating the Bayesian posterior by means of a finite set of $N$ weighted samples,  i.e. particles, $\{(\bm{x}_t^{(i)}, w_t^{(i)})\}_{i=1}^N$, where $w_t^{(i)} \in \mathbb{R}_{\geq 0}$. The PF is a non-parametric approach that approximates the true posterior as
 \begin{align*}
    p\left(\bm{x}_t\right) \approx \sum_{i=1}^N w_t^{(i)} \delta\left(\bm{x}_t - \bm{x}_t^{(i)}\right)
\end{align*}
where $\delta(\cdot)$ denotes the Dirac delta function. It is known that as $N \rightarrow \infty$, the approximated posterior converges to the true posterior \cite{984773}. 

Since we are considering a continuous time SDE as motion model and a discrete-time observation model, we distinguish between two different evolutions of the posterior: 1) the posterior evolves in continuous time in periods $t \in [t_{k-1}, t_k)$ between measurements and 2) a discrete-time Bayesian update step at times $t=t_k$. In the former period, the evolution of the belief is obtained by propagating all particles through the SDE in Eq.~\eqref{eq:SDE}. In the update step, each particle $\bm{x}_{t_k}^{(i)}$ is weighted and resampled according to the observation likelihood $p(\bm{z}_{t_k}\mid \bm{x}^{(i)}_{t_k})$, see, e.g., \cite{nielsen2022state} for more details on continuous-discrete PFs.
\subsection{Risk Measures}
Given a scalar random variable $y \in \mathcal{Y}$ with $p\left(y\right)$ denoting its probability density function (pdf), a risk measure is a function that assigns a real value to the random variable, i.e., $\mathcal{R}: \mathcal{Y} \mapsto \mathbb{R}$.
To evaluate tail events of a random variable, we consider the following risk measures.
\begin{definition}
\label{def:cvar}
The Value-at-Risk ($\mathrm{VaR}$) of a random variable $y \in \mathbb{R}$ with pdf $p\left(y\right)$ at level $\alpha \in (0, 1]$ is given by
\begin{align*}
    \mathrm{VaR}_{\alpha}\left(y\right) &= \underset{\tau \in \mathbb{R}}{\mathrm{inf}} \left\{\tau \mid \mathrm{Pr} \left[y \leq \tau\right] \geq \alpha\right\}.
\end{align*}
    The Conditional-Value-at-Risk (\cvar) of a random variable $y \in \mathbb{R}$ with pdf $p\left(y\right)$ at level $\alpha$ is the expected value of the $\alpha$ quantile, i.e.
    \begin{equation*}
        \mathrm{CVaR}_{\alpha}(y) = \mathbb{E} \left\{y \mid y \leq \mathrm{VaR}_{\alpha}(y)\right\}.
    \end{equation*}
\end{definition}
\var provides a measure that is qualitatively equivalent to chance constraints, i.e.~$\mathrm{VaR}_{\delta}(y) \geq 0 \Leftrightarrow \mathrm{Pr}\left[y \leq 0\right] \leq \delta $.
\cvar, on the other hand, provides a risk measure that considers the mean of the tail of a distribution, which is more risk-averse than $\mathrm{VaR}$. In this work, we choose to use \cvar~as the risk measure suitable for safety-critical systems.
\subsection{Stochastic Control Barrier Functions}
Consider the SDE defined in Eq.~\eqref{eq:SDE} and a safe set 
\begin{equation}
\begin{aligned}
  \label{eq:safeset_state}
    \mathcal{C}_x &= \left\{\bm{x} \in \mathcal{X} \mid h_x\left(\bm{x}\right) \geq 0\right\},\\
    \partial \mathcal{C}_x &= \left\{\bm{x} \in \mathcal{X} \mid h_x\left(\bm{x}\right) = 0\right\}.
 \end{aligned}
\end{equation} 
in which we want the state to remain.  
\begin{definition}
A safe set $\mathcal{C}_x$ is forward invariant with respect to the system \eqref{eq:SDE} if for every initial condition $\bm{x}_{t_0} \in \mathcal{C}_x$
it holds that $\bm{x}_t \in \mathcal{C}_x, \forall t \geq t_0$.
\end{definition}
Stochastic CBFs have been proposed in \cite{clark} which render the safe set $\mathcal{C}_x$ forward invariant assuming that the function $h_x\ofx$ is twice continuously differentiable everywhere. These invariance results can be extended to consider safe sets that are described by non-smooth functions $h_x\ofx$.
    \begin{definition}
    \label{def:NSCBF}
        Given a safe set $\mathcal{C}_x$ defined by \eqref{eq:safeset_state}, the function $B_x\ofx$ serves as a non-smooth stochastic CBF (NSCBF) for the system \eqref{eq:SDE}, if $\forall \bm{x} \in \mathcal{C}_x$
        \begin{enumerate}
            \item There exist class-K functions $\alpha_1$ and $\alpha_2$ such that
            \begin{align}
                \label{eq:SCBF1}
                \frac{1}{\alpha_1\left(h_x\left(\bm{x}\right)\right)} \leq \frac{1}{B_x\ofx} \leq \frac{1}{\alpha_2\left(h_x\left(\bm{x}\right)\right)}.
            \end{align}
            \item There exists a class-K function $\alpha_3$ and a control input $\bm{u} \in \mathcal{U}$ such that
            \begin{align*}
            \frac{\partial B_x}{\partial \bm{x}} \left(\bm{f}\ofx + \bm{g} \ofx \bm{u}\right) + \frac{1}{2} \mathrm{tr}\left[\bm{\sigma}^T \frac{\partial^2 B_x}{\partial \bm{x}^2} \bm{\sigma}\right]\\
            \leq \alpha_3(h_x\ofx).
        \end{align*}
        \end{enumerate}
    \end{definition}
    
    \begin{theorem}
        If a locally Lipschitz control input $\bm{u}(t)$ satisfies Def.~\ref{def:NSCBF} for a given safe set, then $\mathrm{Pr}\left[\bm{x}(t) \in \mathcal{C}_x, \hspace{0.2cm} \forall t\geq t_0\right]=1$, provided that $\bm{x}(t_0) \in \mathcal{C}_x$.\label{thm:NSCBF}
    \end{theorem}
    \begin{proof} The proof can be found in~\cite{vahs2023nonsmooth}.
    \end{proof}
\section{Problem Setting}
We consider robotic systems described by continuous time SDEs in Eq.~\eqref{eq:SDE} and discrete stochastic observation models in Eq.~\eqref{eq:observation}. Ideally, the state of the robot should remain in a safe set $\mathcal{C}_x$ at all times. However, the state is unknown and we only have access to the belief, i.e. an unbounded probability distribution over states. Due to the unboundedness of the belief, there is always the possibility that the state might be outside the safe set, hence we need to formulate safety in terms of risk-awareness.
\begin{figure}[t]
    \centering
    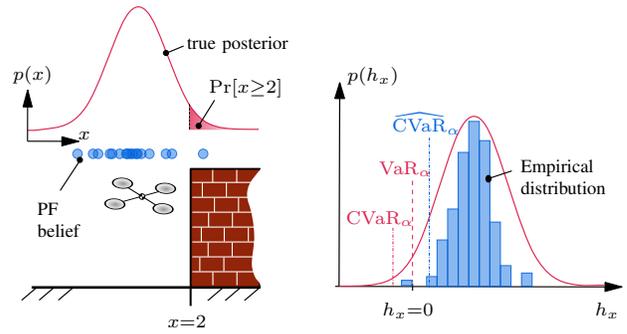
    \caption{\textbf{Left:} A drone with uncertain position $x$ and its pdf $p(x)$ is moving in one dimension. A safety specification is defined as not colliding with the wall, i.e. $h_x = 2 - x$. The true pdf is a Gaussian shown in red and the PF is shown in blue. \textbf{Right: }The distributions over $h_x$ are shown on the right with their corresponding risk measures visualized.}
    \vspace{-0.6cm}
    \label{fig:Example}
\end{figure}

\begin{example}
    \label{exmpl:drone}
    Consider a drone operating in one dimension as shown in Fig.~\ref{fig:Example} in the bottom left, where $x$ denotes its position. A safe set is given as $\mathcal{C}_x = \{x \in \mathbb{R} \mid x \leq 2\}$ which is the set of collision free states. The true posterior follows a Gaussian distribution, i.e. $x \sim \mathcal{N}(\mu, \sigma)$ where $\mu$ and $\sigma$ denote the mean and standard deviation, respectively. The true posterior, is generally unknown and we only have access to the belief, which is modeled by a PF (in our example shown by the blue circles). Consequently, we have a distribution over possible values of $h_x(x) = 2 - x$, as shown in Fig.~\ref{fig:Example} in the right. The true distribution is unbounded (shown by the red curve), whereas the resulting distribution based on the PF belief is given by a histogram distribution. We want to ensure that $\mathrm{CVaR}(h_x(x)) \geq 0$ which is our risk-aware objective.
\end{example}
Although the motivating example is linear and Gaussian, we consider general nonlinear motions and observations, nonlinear safe sets as well as non-Gaussian beliefs. 
\begin{problem}
    Given stochastic motion and observations~\eqref{eq:SDE}-\eqref{eq:observation}, a reference control input $\bm{u}_{\text{ref}}$ and a safe set $\mathcal{C}_x$ defined over states, synthesize control inputs that ensure that at any time $\mathrm{CVaR}_{\alpha}(h_x(\bm{x}(t)) \geq 0$. If exact satisfaction cannot be guaranteed, provide probabilistic satisfaction bounds. \label{prob:control}
\end{problem}
A challenge is that we only know the empirical $\widehat{\mathrm{CVaR}}$ based on the PF, which can mismatch the true \cvar~as not enough particles are available, see Fig.~\ref{fig:Example}. Thus, decisions based on $\widehat{\mathrm{CVaR}}$ might result in unsafe behavior in the real world. In the following, we address this challenge.

\section{Risk-Aware Control}
In the following, we present our solution to Problem~\ref{prob:control} which follows three steps: 1) We define the belief state and the corresponding belief dynamics for continuous-discrete PFs, 2) We construct a  safe set in belief space which, if forward invariant, ensures provable risk-aware safety for the unknown state $\bm{x}$, and, 3)  We formulate a controller that renders the proposed safe set forward invariant, i.e. it ensures that the belief state remains in the safe set.
\subsection{Particle Filter Belief Dynamics}
As the PF used for state estimation consists of a weighted particle set $\{(\bm{x}_t^{(i)}, w_t^{(i)})\}_{i=1}^N$, we define a belief state
\begin{align*}
    \bm{b}_t = \begin{bmatrix}\bm{x}_t^{(1)} & \hdots & \bm{x}_t^{(N)}\end{bmatrix}^T \in \mathbb{R}^{N \cdot n_x}.
\end{align*}
which uniquely describes the current belief at time $t$. 
The continuous time dynamics of the belief state $\bm{b}$ during the period where no observations are available is obtained by propagating each particle through the nonlinear SDE in Eq.~\eqref{eq:SDE}, resulting in
\begin{align}
    \mathrm{d}\bm{b}_t 
    &= \left(\begin{bmatrix}
        \bm{f}\left(\bm{x}_t^{(1)}\right) \\
        \vdots\\
        \bm{f}\left(\bm{x}_t^{(N)}\right)  
    \end{bmatrix} + \begin{bmatrix}
        \bm{g}\left(\bm{x}_t^{(1)}\right) \\
        \vdots\\
        \bm{g}\left(\bm{x}_t^{(N)}\right)  
    \end{bmatrix}\bm{u}\right) \mathrm{d}t + \begin{bmatrix}
    \bm{\sigma} ~\mathrm{d}\bm{W}_t^{(1)}\\
    \vdots\\
    \bm{\sigma}~ \mathrm{d}\bm{W}_t^{(N)}
    \end{bmatrix}\nonumber\\
    &:= \left(\bm{f}_b\left(\bm{b}\right) + \bm{g}_b \left(\bm{b}\right) \bm{u}\right) \mathrm{d}t + \bm{\Sigma} ~ \mathrm{d}\tilde{\bm{W}}\label{eq:BeliefDynamics}
\end{align}
where $\bm{\Sigma} = \mathrm{BD}\left(\{\bm{\sigma}\}_{i=1}^N\right)$ is a block diagonal matrix and $\tilde{\bm{W}}$ is a Brownian motion of dimension $N \cdot q$. 
The belief dynamics are of very high dimension due to the curse of dimensionality in belief spaces which complicates the use of common control techniques such as MPC directly in belief space. 
However, it should be noted that the individual state particles are entirely decoupled which means that the SDE in Eq.~\eqref{eq:BeliefDynamics} can be solved efficiently using parallelization. Further, as $\bm{\Sigma}$ is globally Lipschitz and non-degenerate it is ensured that a solution to Eq.~\eqref{eq:BeliefDynamics} exists.

At discrete times $t_k=t_k^+$ where observations are available, the PF belief is conditioned on the measurement which leads to a discrete update of the belief state, i.e. $\bm{b}(t_k^+) = \bm{\Delta}(\bm{b}(t_k^-), \bm{z}_k)$ where $t_k^-$ is infinitesimal smaller than $t_k^+$. 
This discrete map $\bm{\Delta}(\cdot)$, however, cannot be obtained in closed form as the conditioning includes a stochastic resampling step. In this step, $N$ particles are resampled from the prior belief $\bm{b}(t_k^-)$ according to their likelihood weights. For a detailed description of this step, the reader is referred to \cite{thrun2005probabilistic}. 

\subsection{Construction of Safe Sets}
Now that we have a model of how the belief state $\bm{b}$ changes over time, we construct a safe set $\mathcal{C}_b$ that we want $\bm{b}$ to remain in. This safe set should be defined in a way such that $\bm{b}(t) \in \mathcal{C}_b$ implies risk-awareness for the unknown system state, i.e. $\mathrm{CVaR}_{\alpha}(h_x(\bm{x}(t)) \geq 0$.

The main difficulty lies in the fact that the true distribution $p(h_x(\bm{x}))$ is unknown, thus the exact \cvar~cannot be calculated. Yet, we have access to the PF belief $\bm{b}$ which represents a discrete sample set from the unknown true posterior $p\ofx$. Thus, we can propagate all the particles through the function $h_x \ofx$, i.e. $h_x^{(i)} = h_x(\bm{x}^{(i)})$, to obtain a set of samples from the true unknown distribution $p(h_x\ofx)$. We leverage a theorem that allows us to draw conclusions about the \cvar~of an unknown continuous distribution based on a set of samples that are drawn from that distribution. In contrast to the original version, we reformulate it focusing on the lower-tail distribution.
\begin{theorem}[\cite{thomas2019concentration}, Thm. 3]
Let $y^{(1)}, \dots, y^{(N)}$ be i.i.d. samples from a random variable $y$ and $\mathrm{Pr}\left[y \geq b\right] = 1$ for some finite $b$, then for any $\delta \in (0, 0.5]$, 
\begin{align*}
    \mathrm{Pr}\left[\overline{\mathrm{CVaR}_{\alpha}}\left(y^{(1)},\dots,y^{(N)}\right) \leq \mathrm{CVaR}_{\alpha}\left(y\right) \right] \geq 1 -\delta
\end{align*}
where the lower bound is obtained as
\begin{equation}
\begin{aligned}
    \overline{\mathrm{CVaR}_{\alpha}}\left(y^{(1)},\dots,y^{(N)}\right) &=  \xi_{N+1} + \frac{1}{\alpha} \sum_{i=1}^N \Bigg(\left(\xi_{i} - \xi_{i+1}\right) \cdot\\
    &\left[\frac{i}{N} - \sqrt{\frac{\mathrm{ln}\left(\nicefrac{1}{\delta}\right)}{2N}} - \left(1-\alpha\right)\right]^+\Bigg)\label{eq:cvarbar}
\end{aligned}
\end{equation}
with $\xi_1, \dots, \xi_N$ as order statistics (i.e. $y^{(1)}, \dots, y^{(N)}$ in descending order), $\xi_{N+1} = b$ and $(\cdot)^+ = \mathrm{max}\left\{0, \cdot\right\}$.
\label{thm:cvar}
\end{theorem}
This theorem states that $\overline{\mathrm{CVaR}}$ is a lower bound of the true \cvar with desired probability $1 - \delta$. We leverage this result to construct a safe set in belief space, reading
\begin{align}
    \mathcal{C}_b &= \left\{\bm{b} \in \mathcal{B} \mid h_b(\bm{b}) \geq 0\right\}, \text{where}\nonumber\\
    h_b(\bm{b}) &= \overline{\mathrm{CVaR}_{\alpha}}\left(h_x\left(\bm{x}^{(1)}\right),\dots,h_x\left(\bm{x}^{(N)}\right)\right).\label{eq:cvarbarrier}
\end{align}
In turn, if $\bm{b}(t) \in \mathcal{C}_b \implies \mathrm{Pr}\left[\mathrm{CVaR}(h_x(\bm{x}(t)) \geq 0\right] \geq 1-\delta$, we solve Problem~\ref{prob:control}. The only problem that remains is to design a controller that ensures that the belief state stays within $\mathcal{C}_b$ at all times.

\subsection{Control Synthesis}
Ultimately, we want to guarantee forward invariance of the safe set $\mathcal{C}_b$. 
Popular approaches to render sets forward invariant are CBFs and especially stochastic CBFs that have been designed for stochastic systems as in Eq.~\eqref{eq:BeliefDynamics}. However, the standard formulation of stochastic CBFs cannot be applied here since our safe set defined by Eq.~\eqref{eq:cvarbar} is a continuous but non-smooth function. Due to the sorting operation included in Eq.~\eqref{eq:cvarbar}, the gradient of \cvarbar~with respect to the belief state $\bm{b}$ can be discontinuous. 

This nonsmoothness of the safe set, however, does not affect the forward invariance property, see Thm.~\ref{thm:NSCBF}. Thus, to synthesize control inputs, we solve the quadratic program
\begin{equation}
\begin{aligned}
    \bm{u}^* = \argmin_{\bm{u} \in \mathcal{U}} \quad & \left(\bm{u} - \bm{u}_{\text{ref}}\right)^T \bm{Q}\left(\bm{u} - \bm{u}_{\text{ref}}\right)\\
    \textrm{s.t.~~} \quad & \frac{\partial B_b}{\partial \bm{b}} \left(\bm{f}_b\ofb + \bm{g}_b \ofb \bm{u}\right)\\
    &+ \frac{1}{2} \mathrm{tr}\left[\bm{\Sigma}^T \frac{\partial^2 B_b}{\partial \bm{b}^2} \bm{\Sigma}\right] \leq \gamma h_b\ofb.
\end{aligned}
\label{eq:QP}
\end{equation}
where $B_b\ofb = \nicefrac{1}{h_b\ofb}$, $\gamma \geq 0$, $\bm{Q}$ is a weighting matrix and $\bm{u}_{\text{ref}}$ is a reference control input. In order to provide guarantees for the forward invariance of $\mathcal{C}_b$ under the continuous-discrete evolution of the PF, we need to make the following assumption.
\begin{assumption}
    The belief state $\bm{b}$ does not leave the safe set under a discrete transition $\bm{b}(t_k^+) = \bm{\Delta}(\bm{b}(t_k^-), \bm{z}_k)$, i.e. $\bm{b}(t_k^-) \in \mathcal{C}_b \implies \bm{b}(t_k^+) \in \mathcal{C}_b$.\label{ass:Discrete}
\end{assumption}
\begin{figure*}[t]
\centering
\subfigure[Motivating Example]{
\includegraphics[scale=0.82]{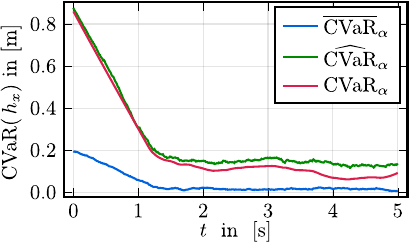}\label{fig:cvarcomparison}}
\subfigure[Safety under multi-modal Uncertainty]{
\includegraphics[scale=0.82]{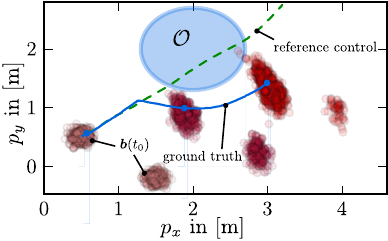}\label{fig:MultimodalSim}}
\subfigure[Hardware Experiments]{
\includegraphics[scale=0.82]{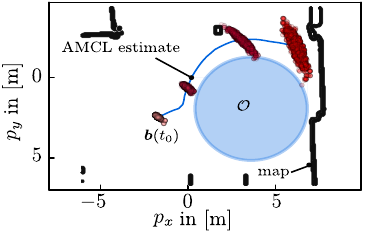}\label{fig:Experiments}}
\caption{\textbf{(a)} Values of the true \cvar, the empirical $\widehat{\mathrm{CVaR}}$ and our lower bound $\overline{\mathrm{CVaR}}$ over time for a single simulation of Example~\ref{exmpl:drone}. \textbf{(b)} Illustration of the PF belief over time for a 2D avoidance task where the initial belief follows a mixture of Gaussians. The ground truth trajectory is shown in blue and the trajectory for the reference controller is shown in green. \textbf{(c)} Illustration of the PF belief and the AMCL estimate for our hardware experiment.}
\vspace{-0.5cm}
\end{figure*}
Since measurements are stochastic and cannot be controlled, this assumption states that we do not receive extreme outlier measurements which make the PF diverge.\footnote{One example of a discrete transition leading to the belief leaving the safe set is when the worst state particle would get resampled $N$ times, meaning that the PF belief would collapse to a single particle. Although the probability of this event happening is non-zero, it is a highly unlikely event.} 
\begin{theorem}
    Under Assumption \ref{ass:Discrete}, if the controller in Eq.~\eqref{eq:QP} is feasible at all times with the nonsmooth stochastic CBF $h_b\ofb$ defined in Eq.~\eqref{eq:cvarbarrier}, then $\mathcal{C}_b$ is forward invariant.
\end{theorem}
\emph{Proof Sketch:} Forward invariance follows directly from Thm.~\ref{thm:NSCBF} in which we replace the state $\bm{x}$ by the belief state $\bm{b}$, the safe set $h_x$ by $h_b$ and the stochastic state dynamics in Eq.~\eqref{eq:SDE} by the stochastic belief dynamics in Eq.~\eqref{eq:BeliefDynamics}.

\medskip

Thus, the proposed controller solves Problem~\ref{prob:control} with desired confidence $1-\delta$ that can be chosen by the user.

\section{Experiments}
We evaluate our approach in simulations as well as hardware experiments in which uncertainties naturally occur. Specifically, we show
\begin{enumerate}
    \item correctness of the risk-aware guarantees in an example where the true posterior is known, in Sec.~\ref{sec:minimalexmpl},
    \item improved adherence to safety requirements over baselines under multi-modal distributions, in Sec.~\ref{sec:multimodal},
    \item integration with existing robot navigation stacks on hardware, in Sec.~\ref{sec:exp}.
\end{enumerate}

{\renewcommand{\arraystretch}{1.15}
\begin{table}[t]
\caption{Comparison of \cvar~measures with varying particles.}
\begin{center}
    \resizebox{0.47\textwidth}{!}{\begin{tabular}{l|ll|ll|ll}
    \multirow{2}{*}{$\bm{N}$} & \multicolumn{2}{c|}{$\bm{\widehat{\mathrm{CVaR}}_{\alpha}}$} & \multicolumn{2}{c|}{$\bm{\overline{\mathrm{CVaR}}_{\alpha}}$} & \multirow{2}{*}{$\bm{t_{c}}~\textbf{in [s]}$} \\ \cline{2-5}
                             & $e(\mu \pm \sigma)$ & $\mathrm{Pr}[e\leq0]$ &   $e(\mu \pm \sigma)$ & $\mathrm{Pr}[e\leq0]$ \\ \hline \hline
    $100$ & \begin{tabular}[c]{@{}c@{}}$-0.025 \pm$ \\ $0.008$\end{tabular} & 100 & \begin{tabular}[c]{@{}c@{}}$0.18 \pm$ \\ $0.022$\end{tabular} & 0.& \begin{tabular}[c]{@{}c@{}}$0.0005 \pm$ \\ $1.3 \cdot 10^{-5}$\end{tabular}\\
    \hline
    $1000$ & \begin{tabular}[c]{@{}c@{}}$-0.001 \pm$ \\ $0.004$\end{tabular}  & 69.5 & \begin{tabular}[c]{@{}c@{}}$0.025 \pm$ \\ $0.005$\end{tabular} & 0& \begin{tabular}[c]{@{}c@{}}$0.0006 \pm$ \\ $2.2 \cdot 10^{-5}$\end{tabular}\\
    \hline
    $5000$ & \begin{tabular}[c]{@{}c@{}}$-0.0004 \pm$ \\ $0.002$\end{tabular} & 53.7 & \begin{tabular}[c]{@{}c@{}}$ 0.015 \pm$ \\ $0.018$\end{tabular} & 0.& \begin{tabular}[c]{@{}c@{}}$0.001 \pm$ \\ $0.003$\end{tabular}
    \end{tabular}}
    \end{center}
    \label{tb:cvars}
    \vspace{-0.8cm}
\end{table}}
\subsection{Motivating Example Continued}
\label{sec:minimalexmpl}
\subsubsection{Setup}In this simulation we refer back to our motivating example shown in Fig.~\ref{fig:Example}. The drone's dynamics are described by a single integrator $\text{d}x = u ~\text{d}t + 0.1 ~\text{d}W$ with a reference controller as $u_{\text{ref}}=1$, i.e. a controller that would steer the drone into the wall. In this simulation we do not consider observations, so the uncertainty will increase over time. As the initial belief is Gaussian and the dynamics are linear, the true posterior is in this case given by a Kalman Filter (KF). We use the KF to calculate the mismatch of the \cvar~between the true posterior and the approximated PF belief. We evaluate the mismatch as $e=\mathrm{CVaR} - \phi$ where $\phi \in \{\widehat{\mathrm{CVaR}}, \overline{\mathrm{CVaR}}\}$. Note that $e <0$ means that \cvar~is overapproximated which should be avoided at all times. For this simulation we picked $\alpha = 0.2$ (thus we evaluate the mean of the worst 20 \% quantile) and $\delta = 0.05$ which means that our \cvarbar measure should be underapproximating in 95~\% of the cases.

\subsubsection{Discussion}Table~\ref{tb:cvars} shows the error $e$ for varying numbers of particles $N$. It can be seen that the empirical $\widehat{\mathrm{CVaR}}$ generally overapproximates the true \cvar~which stresses that empirical measures should not be used for safety-critical systems. Our \cvarbar~measure on the other hand is a lower bound in all cases which means that the bound in Thm.~\ref{thm:cvar} might not be very tight. However, the error $e$ decreases with the number of particles as the PF belief approaches the true belief. Fig.~\ref{fig:cvarcomparison} shows the evolution of the different \cvar~measures over time for $N=100$. It can be seen that our controller keeps $h_b\ofb \geq 0$, thus, rendering the set $\mathcal{C}_b$ forward invariant.

Additionally, Table~\ref{tb:cvars} shows the computation times for solving the QP in Eq.~\eqref{eq:QP}. Since we are only solving a QP over control inputs, the controller can still be run at a rate of almost 1kHz for a belief dimension of 5000 which shows that it is a real-time applicable approach.
\subsection{Safety under multi-modal Uncertainty}
\label{sec:multimodal}
\subsubsection{Setup}
We compare our proposed method to various baselines in an example including nonlinear dynamics and observations as well as a non-Gaussian initial belief. We consider a unicycle robot that uses noisy radio signals for localization which can be described by
\begin{equation}
\begin{aligned}
\begin{bmatrix}
    \text{d}{p}_x\\
    \text{d}{p}_y\\
    \text{d}{\varphi}
\end{bmatrix} = \begin{bmatrix}
    \mathrm{c}\varphi & 0\\
    \mathrm{c}\varphi & 0\\
    0 & \mathrm{s}\varphi
\end{bmatrix} \bm{u} + \bm{\sigma}~\mathrm{d}\bm{W},\hspace{0.2cm} \begin{array}{l}
    z_k = \lVert \bm{p} - \bm{\ell} \rVert_2 + v_k,\\
    v_k \sim \mathcal{N}\left(0, r)\right).
\end{array}
\end{aligned}
\label{eq:Unicycle}
\end{equation}
where $\mathrm{s}\varphi=\mathrm{sin}(\varphi)$, $\mathrm{c}\varphi=\mathrm{cos}(\varphi)$, $\bm{x}=[p_x, p_y, \varphi]^T$ is the state with $\bm{p}=[p_x, p_y]^T$ being the 2D position and $\bm{u}=[v, \omega]^T$ is the control input. The antenna from which we obtain noisy distance measurements $z$ is located at $\bm{\ell}=[4, 4]^T \text{~m}$. The motion noise is $\bm{\sigma}=\mathrm{diag}\left(\left[0.3, 0.3, 0.1\right]\right)$ and observation noise $r = 0.3$ with a sensor update rate of 1~Hz. The initial belief is given as a mixture of Gaussians $p(\bm{x}(t_0))=0.5 \cdot \mathcal{N}(\bm{p}_1, \bm{\Sigma}_1) + 0.5 \cdot \mathcal{N}(\bm{p}_2, \bm{\Sigma}_2)$ which is visualized in Fig.~\ref{fig:MultimodalSim}. A proportional controller that drives the robot to a goal based on the mean state of the PF belief serves as reference controller $\bm{u}_{\text{ref}}$, see Fig.~\ref{fig:MultimodalSim}.

A safe set which enables smooth navigation for unicycle dynamics with circular stay-out regions (from \cite{huang2023obstacle}) is defined as $h_x\ofx = \lVert \hat{\bm{p}} - \mathcal{O}\rVert_2 - (r_o + d)$ where
\begin{align}
    \hat{\bm{p}} = \begin{bmatrix}
        p_x + d \cdot \mathrm{cos}(\varphi)\\
        p_y + d \cdot \mathrm{sin}(\varphi)
    \end{bmatrix}
\end{align}
with $d>0$. This function is chosen because it allows to control both velocities, $v$ and $\omega$, simultaneously. For risk-aware control, we embed $h_x\ofx$ in our safe set defined over beliefs in Eq.~\eqref{eq:cvarbarrier} and solve the QP in Eq.~\eqref{eq:QP}.
\subsubsection{Baselines}
We compare our approach to three different baselines that are based on stochastic CBFs (SCBFs): 1.) A standard SCBF called $\mu$-SCBF that uses the mean state, i.e. the weighted sum of all particles. 2.) A standard SCBF called ML-SCBF based on the most likely particle, i.e. the particle with the highest observation likelihood. 3.) An SCBF that considers a known time-varying estimation error. Here we use Chebyshev inequality \cite{chen2007new} to obtain a ball $\mathcal{B}_{\eta}=\{\bm{x}\in \mathcal{X}|\lVert \bm{x}-\bm{\mu}\rVert_2^2 \leq \nicefrac{\mathrm{tr}[\bm{\Sigma}]}{\eta}\}$ which is centered around the mean $\bm{\mu}$ s.t. $\mathrm{Pr}[\bm{p} \in \mathcal{B}_{\eta}] \geq 1 - \eta$ with $\eta = 0.05$.
{\renewcommand{\arraystretch}{1.2}
\begin{table}[t]
    \caption{Comparison to baselines.}
     \resizebox{0.47\textwidth}{!}{\begin{tabular}{l | ccccc}
   ~ & $\bm{\mu}$\bf{-SCBF} & \bf{ML-SCBF} & \bf{BE-SCBF} & \begin{tabular}[c]{@{}c@{}}\bf{Ours} \\ ($\alpha = 0.2$)\end{tabular} & \begin{tabular}[c]{@{}c@{}}\bf{Ours} \\ ($\alpha = 0.05$)\end{tabular} \\
   \hline\hline
    \# coll.  &     46      &   42    &      0    &      2  & 0 \\
    \hline
    \begin{tabular}[c]{@{}c@{}}$h_x(\bm{x})$ \\ $(\mu \pm \sigma)$\end{tabular}   &      \begin{tabular}[c]{@{}c@{}}$0.85 \pm$ \\ $0.59$\end{tabular}    &   \begin{tabular}[c]{@{}c@{}}$0.8 \pm$ \\ 0.56\end{tabular}  &    \begin{tabular}[c]{@{}c@{}}$2.45 \pm$ \\ $0.49$\end{tabular}     &    \begin{tabular}[c]{@{}c@{}}$0.95 \pm$ \\ 0.46\end{tabular} & \begin{tabular}[c]{@{}c@{}}$1.71 \pm$ \\ 0.43\end{tabular}
\end{tabular}}
\label{tb:baselines}
\vspace{-0.5cm}
\end{table}}
\subsubsection{Discussion}
Figure \ref{fig:MultimodalSim} shows an examplary trajectory of the PF for the avoidance task using our proposed method. It can be clearly seen that our method corrects the reference controller such that the majority of particles avoid the area $\mathcal{O}$. In this simulation, we used $\mathrm{CVaR}_{0.2}$, i.e. we look at the mean of the worst 20 \% quantile. We used $N=1000$ particles and repeated the simulation 100 times. Table~\ref{tb:baselines} summarizes the results of our proposed method as well as the baselines. The two baselines $\mu$-SCBF and ML-SCBF fail to ensure safety in almost half of the cases which is due to the fact that they do not consider the entire distribution. On the other hand, BE-SCBF achieves zeros safety violations but is overly conservative as the distance to the boundary of the safe set ($h_x$) is more than twice as high as for our method. This is because reliable estimation errors are difficult to obtain, thus Chebyshev inequality is rather conservative if the distribution is not Gaussian. Our method ensures risk-aware safety where the level of riskiness can be defined by the user. Note that, since $\mathrm{CVaR}_{\alpha}(h_x) \leq \mathrm{VaR}_{\alpha}(h_x)$ it is also ensured that $\mathrm{CVaR}_{\alpha}(h_x) \geq 0 \implies \mathrm{Pr}[h_x \geq 0 ] \geq 1 - \alpha$ which is satisfied in our simulations.
\subsection{Hardware Experiments}
 We implemented our approach as a ROS2 node which subscribes to a PF belief as well as a reference control input and publishes a minimally deviating control input that ensures risk-aware safety.  The package is open-source for anyone who wishes to make it a part of their navigation.
\label{sec:exp}
\subsubsection{Setup}
In our hardware experiments, we used the ABB Mobile YuMi Research Platform in the WASP Research Arena for Robotics (WARA Robotics) at ABB, depicted in Fig.~\ref{fig:Intro}. The robot is running the ROS2 navigation stack (NAV2 \cite{macenski2020marathon}) that enables functionalities such as localization and planning. For state estimation, two LiDARs are used to localize the robot in a given map using Adaptive Monte Carlo Localization (AMCL) which is a PF approach with varying number of particles. The robots dynamics are modeled as
\begin{align*}
    \begin{bmatrix}
        \text{d}{p}_x\\
        \text{d}{p}_y\\
        \text{d}\varphi
    \end{bmatrix}&=\begin{bmatrix}
        \mathrm{c}\varphi & - \mathrm{s}\varphi & 0\\
        \mathrm{s}\varphi & \mathrm{c}\varphi & 0\\
        0 & 0 & 1
    \end{bmatrix}\bm{u}~\text{d}t +\bm{\sigma}~\text{d}\bm{W}
\end{align*}
where $\bm{u} = [v_x, v_y, \omega]^T$ are the control inputs consisting of a reference linear and angular velocity in the robot's local frame. The observation model is a likelihood field model which is the standard observation model for LiDAR localization in the NAV2 stack. Our node is subscribing to the AMCL PF belief as well as the velocity command provided by the local planner. In the AMCL node, the particle size varies between $3000-4000$ particles and our risk-aware controller publishes a control signal at a rate of 30 Hz.

We specify our safe set as a circular area, shown in Fig.~\ref{fig:Experiments}, that we want to avoid , i.e. $h_x(\bm{x}) = \lVert \bm{p} - \mathcal{O}\rVert_2 - r_o$. As we can directly control the velocity in $p_x$ and $p_y$ direction, the modification used in the previous section is not necessary.

\subsubsection{Discussion}We investigate a scenario in which a sensor failure happens, i.e. lidar measurements are not available and the robot purely relies on odometry information. Figure~\ref{fig:Experiments} shows the PF belief as well as the AMCL estimate over time. It can be seen that the majority of particles stays withing the safe area as enforced by our controller. Further, it is interesting to see that the distance of the AMCL estimate to $\mathcal{O}$ gradually increases since the uncertainty increases\footnote{video available at \href{https://www.youtube.com/watch?v=-YUdRtdFbYc}{https://www.youtube.com/watch?v=-YUdRtdFbYc}}. In this experiment, we cannot evaluate the actual distance to the safe set as ground truth data is not available.

\section{CONCLUSIONS AND FUTURE WORK}
Our work enables risk-aware safety for robots with nonlinear stochastic dynamics and observations by leveraging particle filters and control barrier functions. We define safe sets using Conditional-Value-at-Risk underapproximations that account for the mismatch between the approximated belief and the true unknown posterior which has not been considered before. We offer a computationally efficient control technique which we have analyzed in simulations as well as real world experiments, showcasing improved adherence to safety specifications over baselines. In future work, we aim to use our results for safe navigation in dynamic environments with probabilistic trajectory forecasting models.


\balance
\bibliographystyle{IEEEtran}
\bibliography{refs.bib}

\begin{thebibliography}{10}
\providecommand{\url}[1]{#1}
\csname url@samestyle\endcsname
\providecommand{\newblock}{\relax}
\providecommand{\bibinfo}[2]{#2}
\providecommand{\BIBentrySTDinterwordspacing}{\spaceskip=0pt\relax}
\providecommand{\BIBentryALTinterwordstretchfactor}{4}
\providecommand{\BIBentryALTinterwordspacing}{\spaceskip=\fontdimen2\font plus
\BIBentryALTinterwordstretchfactor\fontdimen3\font minus \fontdimen4\font\relax}
\providecommand{\BIBforeignlanguage}[2]{{%
\expandafter\ifx\csname l@#1\endcsname\relax
\typeout{** WARNING: IEEEtran.bst: No hyphenation pattern has been}%
\typeout{** loaded for the language `#1'. Using the pattern for}%
\typeout{** the default language instead.}%
\else
\language=\csname l@#1\endcsname
\fi
#2}}
\providecommand{\BIBdecl}{\relax}
\BIBdecl

\bibitem{thrun2005probabilistic}
S.~Thrun, W.~Burgard, and D.~Fox, \emph{Probabilistic robotics}.\hskip 1em plus 0.5em minus 0.4em\relax Cambridge, Mass.: MIT Press, 2005.

\bibitem{dellaert1999monte}
``Monte carlo localization for mobile robots,'' in \emph{International Conference on Robotics and Automation (ICRA)}, vol.~2.\hskip 1em plus 0.5em minus 0.4em\relax IEEE, 1999, pp. 1322--1328.

\bibitem{blackmore2006probabilistic}
L.~Blackmore, H.~Li, and B.~Williams, ``A probabilistic approach to optimal robust path planning with obstacles,'' in \emph{American Control Conference}.\hskip 1em plus 0.5em minus 0.4em\relax IEEE, 2006, pp. 7--pp.

\bibitem{van2012motion}
J.~Van Den~Berg, S.~Patil, and R.~Alterovitz, ``Motion planning under uncertainty using iterative local optimization in belief space,'' \emph{The International Journal of Robotics Research}, vol.~31, no.~11, pp. 1263--1278, 2012.

\bibitem{8613928}
H.~Zhu and J.~Alonso-Mora, ``Chance-constrained collision avoidance for mavs in dynamic environments,'' \emph{IEEE Robotics and Automation Letters}, vol.~4, no.~2, pp. 776--783, 2019.

\bibitem{ono2008iterative}
M.~Ono and B.~C. Williams, ``Iterative risk allocation: A new approach to robust model predictive control with a joint chance constraint,'' in \emph{Conference on Decision and Control}.\hskip 1em plus 0.5em minus 0.4em\relax IEEE, 2008, pp. 3427--3432.

\bibitem{Zhang_2022}
Y.~Zhang, S.~Walters, and X.~Xu, ``Control barrier function meets interval analysis: Safety-critical control with measurement and actuation uncertainties,'' in \emph{American Control Conference ({ACC})}.\hskip 1em plus 0.5em minus 0.4em\relax {IEEE}, 2022.

\bibitem{Wang_2021}
Y.~Wang and X.~Xu, ``Observer-based control barrier functions for safety critical systems,'' in \emph{American Control Conference (ACC)}.\hskip 1em plus 0.5em minus 0.4em\relax IEEE, 2022, pp. 709--714.

\bibitem{dean_20}
S.~Dean, A.~Taylor, R.~Cosner, B.~Recht, and A.~Ames, ``Guaranteeing safety of learned perception modules via measurement-robust control barrier functions,'' in \emph{Conference on Robot Learning}.\hskip 1em plus 0.5em minus 0.4em\relax PMLR, 2021, pp. 654--670.

\bibitem{de2021scenario}
O.~de~Groot, B.~Brito, L.~Ferranti, D.~Gavrila, and J.~Alonso-Mora, ``Scenario-based trajectory optimization in uncertain dynamic environments,'' \emph{IEEE Robotics and Automation Letters}, vol.~6, no.~3, pp. 5389--5396, 2021.

\bibitem{de2023scenario}
O.~de~Groot, L.~Ferranti, D.~Gavrila, and J.~Alonso-Mora, ``Scenario-based motion planning with bounded probability of collision,'' \emph{arXiv preprint arXiv:2307.01070}, 2023.

\bibitem{nyberg2021risk}
T.~Nyberg, C.~Pek, L.~Dal~Col, C.~Nor{\'e}n, and J.~Tumova, ``Risk-aware motion planning for autonomous vehicles with safety specifications,'' in \emph{2021 IEEE Intelligent Vehicles Symposium (IV)}, 2021, pp. 1016--1023.

\bibitem{barbosa2021risk}
F.~S. Barbosa, B.~Lacerda, P.~Duckworth, J.~Tumova, and N.~Hawes, ``Risk-aware motion planning in partially known environments,'' in \emph{60th IEEE Conference on Decision and Control (CDC)}, 2021, pp. 5220--5226.

\bibitem{majumdar2020should}
A.~Majumdar and M.~Pavone, ``How should a robot assess risk? towards an axiomatic theory of risk in robotics,'' in \emph{Robotics Research: The 18th International Symposium ISRR}.\hskip 1em plus 0.5em minus 0.4em\relax Springer, 2020, pp. 75--84.

\bibitem{althoff2014online}
M.~Althoff and J.~M. Dolan, ``Online verification of automated road vehicles using reachability analysis,'' \emph{IEEE Transactions on Robotics}, vol.~30, no.~4, pp. 903--918, 2014.

\bibitem{pek2020using}
C.~Pek, S.~Manzinger, M.~Koschi, and M.~Althoff, ``Using online verification to prevent autonomous vehicles from causing accidents,'' \emph{Nature Machine Intelligence}, vol.~2, no.~9, pp. 518--528, 2020.

\bibitem{campo1987robust}
P.~J. Campo and M.~Morari, ``Robust model predictive control,'' in \emph{1987 American control conference}.\hskip 1em plus 0.5em minus 0.4em\relax IEEE, 1987, pp. 1021--1026.

\bibitem{langson2004robust}
W.~Langson, I.~Chryssochoos, S.~Rakovi{\'c}, and D.~Q. Mayne, ``Robust model predictive control using tubes,'' \emph{Automatica}, vol.~40, no.~1, pp. 125--133, 2004.

\bibitem{roque2022corridor}
P.~Roque, W.~S. Cortez, L.~Lindemann, and D.~V. Dimarogonas, ``Corridor mpc: Towards optimal and safe trajectory tracking,'' in \emph{American Control Conference (ACC)}.\hskip 1em plus 0.5em minus 0.4em\relax IEEE, 2022, pp. 2025--2032.

\bibitem{zeilinger2009real}
M.~N. Zeilinger, C.~N. Jones, D.~M. Raimondo, and M.~Morari, ``Real-time mpc-stability through robust mpc design,'' in \emph{Conference on Decision and Control (CDC)}.\hskip 1em plus 0.5em minus 0.4em\relax IEEE, 2009, pp. 3980--3986.

\bibitem{xu2015robustness}
X.~Xu, P.~Tabuada, J.~W. Grizzle, and A.~D. Ames, ``Robustness of control barrier functions for safety critical control,'' \emph{IFAC-PapersOnLine}, vol.~48, no.~27, pp. 54--61, 2015.

\bibitem{kolathaya2018input}
S.~Kolathaya and A.~D. Ames, ``Input-to-state safety with control barrier functions,'' \emph{IEEE control systems letters}, vol.~3, no.~1, pp. 108--113, 2018.

\bibitem{schwarm1999chance}
A.~T. Schwarm and M.~Nikolaou, ``Chance-constrained model predictive control,'' \emph{AIChE Journal}, vol.~45, no.~8, pp. 1743--1752, 1999.

\bibitem{clark}
A.~Clark, ``Control barrier functions for complete and incomplete information stochastic systems,'' in \emph{American Control Conference (ACC)}, 2019, pp. 2928--2935.

\bibitem{cosner2023robust}
R.~K. Cosner, P.~Culbertson, A.~J. Taylor, and A.~D. Ames, ``Robust safety under stochastic uncertainty with discrete-time control barrier functions,'' \emph{arXiv preprint arXiv:2302.07469}, 2023.

\bibitem{black2023safety}
M.~Black, G.~Fainekos, B.~Hoxha, D.~Prokhorov, and D.~Panagou, ``Safety under uncertainty: Tight bounds with risk-aware control barrier functions,'' \emph{arXiv preprint arXiv:2304.01040}, 2023.

\bibitem{li2002probabilistically}
P.~Li, M.~Wendt, and G.~Wozny, ``A probabilistically constrained model predictive controller,'' \emph{Automatica}, vol.~38, no.~7, pp. 1171--1176, 2002.

\bibitem{oldewurtel2013stochastic}
F.~Oldewurtel, C.~N. Jones, A.~Parisio, and M.~Morari, ``Stochastic model predictive control for building climate control,'' \emph{IEEE Transactions on Control Systems Technology}, vol.~22, no.~3, pp. 1198--1205, 2013.

\bibitem{singletary2022safe}
A.~Singletary, M.~Ahmadi, and A.~D. Ames, ``Safe control for nonlinear systems with stochastic uncertainty via risk control barrier functions,'' \emph{IEEE Control Systems Letters}, vol.~7, pp. 349--354, 2022.

\bibitem{lew2023risk}
T.~Lew, R.~Bonalli, and M.~Pavone, ``Risk-averse trajectory optimization via sample average approximation,'' \emph{arXiv preprint arXiv:2307.03167}, 2023.

\bibitem{vahs2023belief}
M.~Vahs, C.~Pek, and J.~Tumova, ``Belief control barrier functions for risk-aware control,'' \emph{Robotics and Automation Letters}, 2023.

\bibitem{ostafew2016robust}
C.~J. Ostafew, A.~P. Schoellig, and T.~D. Barfoot, ``Robust constrained learning-based nmpc enabling reliable mobile robot path tracking,'' \emph{The International Journal of Robotics Research}, vol.~35, no.~13, pp. 1547--1563, 2016.

\bibitem{blackmore2010probabilistic}
L.~Blackmore, M.~Ono, A.~Bektassov, and B.~C. Williams, ``A probabilistic particle-control approximation of chance-constrained stochastic predictive control,'' \emph{IEEE transactions on Robotics}, vol.~26, no.~3, pp. 502--517, 2010.

\bibitem{platt2017efficient}
R.~Platt, L.~Kaelbling, T.~Lozano-Perez, and R.~Tedrake, ``Efficient planning in non-gaussian belief spaces and its application to robot grasping,'' in \emph{Robotics Research}.\hskip 1em plus 0.5em minus 0.4em\relax Springer, 2017, pp. 253--269.

\bibitem{sunberg2018online}
Z.~Sunberg and M.~Kochenderfer, ``Online algorithms for pomdps with continuous state, action, and observation spaces,'' in \emph{Proceedings of the International Conference on Automated Planning and Scheduling}, vol.~28, 2018, pp. 259--263.

\bibitem{leobacher2017strong}
G.~Leobacher and M.~Sz{\"o}lgyenyi, ``A strong order 1/2 method for multidimensional sdes with discontinuous drift,'' \emph{The Annals of Applied Probability}, vol.~27, no.~4, pp. 2383--2418, 2017.

\bibitem{984773}
D.~Crisan and A.~Doucet, ``A survey of convergence results on particle filtering methods for practitioners,'' \emph{IEEE Transactions on Signal Processing}, vol.~50, no.~3, pp. 736--746, 2002.

\bibitem{nielsen2022state}
M.~K. Nielsen, T.~K. Ritschel, I.~Christensen, J.~Dragheim, J.~K. Huusom, K.~V. Gernaey, and J.~B. J{\o}rgensen, ``State estimation for continuous-discrete-time nonlinear stochastic systems,'' \emph{arXiv preprint arXiv:2212.02139}, 2022.

\bibitem{vahs2023nonsmooth}
M.~Vahs and J.~Tumova, ``Non-smooth control barrier functions for stochastic dynamical systems,'' \emph{Accepted at European Control Conference (ECC)}, 2024.

\bibitem{thomas2019concentration}
P.~Thomas and E.~Learned-Miller, ``Concentration inequalities for conditional value at risk,'' in \emph{International Conference on Machine Learning}.\hskip 1em plus 0.5em minus 0.4em\relax PMLR, 2019, pp. 6225--6233.

\bibitem{huang2023obstacle}
J.~Huang, Z.~Liu, J.~Zeng, X.~Chi, and H.~Su, ``Obstacle avoidance for unicycle-modelled mobile robots with time-varying control barrier functions,'' 2023.

\bibitem{chen2007new}
X.~Chen, ``A new generalization of chebyshev inequality for random vectors,'' \emph{arXiv preprint arXiv:0707.0805}, 2007.

\bibitem{macenski2020marathon}
S.~Macenski, F.~Mart{\'\i}n, R.~White, and J.~G. Clavero, ``The marathon 2: A navigation system,'' in \emph{International Conference on Intelligent Robots and Systems (IROS)}.\hskip 1em plus 0.5em minus 0.4em\relax IEEE, 2020, pp. 2718--2725.

\end{thebibliography}

\end{document}